\documentclass[letterpaper, 10 pt, conference]{ieeeconf}  

\IEEEoverridecommandlockouts                              

\usepackage{amsfonts,latexsym,amsmath,amssymb}
\usepackage[mathscr]{eucal}
\usepackage{graphicx,color}
\usepackage{mathtools}
\usepackage{stmaryrd}
\usepackage{url}
\usepackage{dsfont}
 
\usepackage{amsthm}
\usepackage{bbm}
\usepackage{psfrag}

\usepackage{amsgen}
\usepackage{amstext}
\usepackage{amsbsy}
\usepackage{amsopn}

\usepackage{bm}

\providecommand{\R}{\mathbb{R}}

\providecommand{\S}{\mathrm{S}}

\providecommand{\SO}{\mathbf{SO}}

\providecommand{\so}{\mathfrak{so}}

\providecommand{\calB}{\mathcal{B}}
\providecommand{\calC}{\mathcal{C}}

\providecommand{\calI}{\mathcal{I}}

\providecommand{\calS}{\mathcal{S}}

\DeclareMathOperator{\grad}{grad}

\DeclareMathOperator{\diag}{diag}

\providecommand{\tT}{\mathrm{T}} 
\providecommand{\td}{\mathrm{d}}
\providecommand{\tD}{\mathrm{D}}

\providecommand{\ddt}{\frac{\td}{\td t}}

\usepackage{accents}
\makeatletter
\providecommand{\scirc}{%
    \hbox{\fontfamily{\rmdefault}\fontsize{0.4\dimexpr(\f@size pt)}{0}\selectfont{\raisebox{-0.52ex}[0ex][-0.52ex]{$\circ$}}}}

\makeatother

\mathchardef\mhyphen="2D

\def \S {{\mathrm S}}

\DeclareMathOperator{\sign}{sign}
\DeclareMathOperator{\rank}{rank}
\usepackage[colorlinks,bookmarksopen,bookmarksnumbered,citecolor=red,urlcolor=red]{hyperref}

\makeatletter
\newcommand{\pushright}[1]{\ifmeasuring@#1\else\omit\hfill$\displaystyle#1$\fi\ignorespaces}
\newcommand{\pushleft}[1]{\ifmeasuring@#1\else\omit$\displaystyle#1$\hfill\fi\ignorespaces}
\makeatother

\newtheorem{theorem}{Theorem}

\newtheorem{definition}{Definition}
\newtheorem{assumption}{Assumption}
\newtheorem{lemma}{Lemma}
\newtheorem{remark}{Remark}

\usepackage{siunitx}

\title{\LARGE \bf
Observer Design for Optical Flow-Based Visual-Inertial Odometry with Almost-Global Convergence
(Extended Version)}
\author{Tarek Bouazza, Soulaimane Berkane, \IEEEmembership{Senior Member, IEEE}, Minh-Duc Hua, \\ and Tarek Hamel, \IEEEmembership{Fellow Member, IEEE}
\thanks{This work was supported by the ``Grands Fonds Marins'' Project Deep-C, the ASTRID ANR project ASCAR, and also in part by NSERC-DG RGPIN-2020-04759 and Fonds de recherche du Qu\'ebec (FRQ).
}
\thanks{T. Bouazza, M-D. Hua, and T. Hamel are with I3S, CNRS, Université Côte d'Azur, 06903 Sophia Antipolis, France, T. Hamel is also with Institut Universitaire de France (IUF),
        {\tt\small \{bouazza,hua,thamel\}@i3s.unice.fr}.}
\thanks{S. Berkane is with the Département d’informatique et d’ingénierie, Université du Québec en Outaouis, Gatineau, QC J8X 3X7, Canada, and also with the Department of Electrical
Engineering, Lakehead University, Thunder Bay, ON P7B 5E1, Canada,
{\tt\small soulaimane.berkane@uqo.ca}.}
}

\begin{document}

\maketitle
\thispagestyle{empty}
\pagestyle{empty}

\begin{abstract}
This paper presents a novel cascaded observer architecture that combines optical flow and IMU measurements to perform continuous monocular visual-inertial odometry (VIO). The proposed solution estimates body-frame velocity and gravity direction simultaneously by fusing velocity direction information from optical flow measurements with gyro and accelerometer data. This fusion is achieved using a globally exponentially stable Riccati observer, which operates under persistently exciting translational motion conditions. The estimated gravity direction in the body frame is then employed, along with an optional magnetometer measurement, to design a complementary observer on $\SO(3)$ for attitude estimation. The resulting interconnected observer architecture is shown to be almost globally asymptotically stable.
To extract the velocity direction from sparse optical flow data, a gradient descent algorithm is developed to solve a constrained minimization problem on the unit sphere. The effectiveness of the proposed algorithms is validated through simulation results.
\end{abstract}

    \section{Introduction}
    Over the past decades, the fusion of monocular vision with inertial measurements from MEMS-based Inertial Measurement Units (IMUs) has become a cornerstone of robotics for estimating the navigation states of a robot, namely its position, orientation and linear velocity. The complementary nature of the two sensing modalities makes visual-inertial systems a reliable alternative to Global Navigation Satellite Systems (GNSS), which can be ineffective in indoor environments.
    The classical problem of visual-inertial odometry (VIO), which involves estimating a robot’s trajectory by combining inertial and visual information, is a key component of autonomous navigation for aerial, ground, and underwater vehicles. 
    VIO is commonly considered a special case of the broader visual-inertial Simultaneous Localization and Mapping (SLAM) problem, which also involves reconstructing a map of the unknown surrounding environment.

    Monocular vision-only and visual-inertial odometry algorithms typically fall into two main categories: intensity-based (or direct) and feature-based methods. 
    Direct methods estimate motion directly from image intensity changes across consecutive images, such as SVO \cite{forster2014svo}, and DTAM \cite{newcombe2011dtam}, while feature-based approaches track a sparse set of keypoints across frames. 
    Existing feature-based VIO can be further divided into optimization-based and Extended Kalman Filter (EKF)-based approaches \cite{scaramuzza2019visual}. Optimization-based methods, such as VINS-Mono \cite{qin2018vins} and OKVIS \cite{leutenegger2013keyframe}, refine the trajectory by jointly minimizing a cost function over multiple frames and tend to achieve high accuracy at the expense of computational complexity \cite{delmerico2018benchmark}. In contrast, EKF-based methods, such as MSCKF \cite{mourikis2007multi} and ROVIO \cite{bloesch2015robust}, provide a state estimation framework that fuses visual-inertial measurements in real time, 
    resulting in computationally efficient algorithms but often at the cost of reduced accuracy due to linearization errors.
    More recently, nonlinear observers designed on matrix Lie groups have been largely adopted for VIO and SLAM to achieve better performance and strong stability guarantees \cite{barrau2015ekf,mahony2017geometric,van2023eqvio,fornasier2023msceqf}.
    EKF-based VIO, and more broadly EKF-based SLAM, when formulated in an inertial frame, is not fully observable due to its inherent invariance to the choice of the reference frame \cite{andrade2004effects}. Without additional measurements, this invariance results in an unobservable subspace corresponding to the inertial position and yaw rotation of the inertial frame \cite{van2023eqvio}.
    
    Unlike feature-based methods that rely on tracking discrete features across frames, and either using the epipolar constraint \cite{ma2004invitation} to compute inter-frame motion \cite{scaramuzza2011visual}, or including them into the state estimator, an alternative formulation consists of using a continuous representation by exploiting the optical flow (or displacement field) of the features \cite{jaegle2016fast} to obtain a velocity direction vector. 
    This \textit{continuous} approach simplifies the VIO problem by directly providing instantaneous velocity information rather than discrete features, which is inherently more suited for odometry tasks.

This paper presents a novel observer architecture for visual-inertial odometry (VIO) that leverages optical flow and IMU measurements to estimate body-frame velocity and gravity direction. The proposed approach addresses the challenge of partial observability of the navigation states in the inertial frame by focusing on the independent estimation of body-frame velocity and gravity direction.
Specifically, we introduce a globally exponentially stable Riccati observer, which fuses velocity direction information from sparse optical flow data with gyro and accelerometer measurements. This observer operates effectively under sufficient persistence of excitation, ensuring robust estimation of the body-frame velocity and gravity direction.

The estimated body-frame gravity vector is then employed in a second-stage nonlinear observer on the Lie group $\SO(3)$, designed to estimate the reduced attitude (pitch and roll) by exploiting gravity as a global reference direction. Since gravity alone only provides the orientation up to an unknown yaw angle, we incorporate an additional body-frame direction measurement, such as the magnetic field, when available, to recover the full attitude. This cascaded estimation framework ensures the decoupling of velocity and attitude estimation, while guaranteeing the almost global asymptotic stability of the interconnection error dynamics.

A complementary key contribution of this work is the development of a pre-observer stage, which employs a novel gradient descent algorithm to estimate the translational velocity direction from sparse optical flow measurements. 
Unlike conventional methods that formulate this problem as a least-squares optimization in $\mathbb{R}^3$ \cite{zhang1999fast,jaegle2016fast}, our method enforces the unit-norm constraint by solving a constrained minimization problem directly on the unit sphere $\mathbb{S}^2$. 
This strategy ensures consistency with the underlying geometry and provides a reliable estimate of the velocity direction required for the proposed observer.

The remainder of the paper is organized as follows. Section \ref{sec:prelims} introduces the preliminary notation and uniform observability definitions. Section \ref{sec:problem_desc} formally states the problem at hand, and Section \ref{sec:observer} then outlines the key stages of the proposed observer design and presents the stability analysis. In Section \ref{sec:opt_flow_alg}, we present an iterative algorithm to estimate the velocity direction using optical flow measurements. Finally, Section \ref{sec:simresults} discusses simulation results, followed by concluding remarks in Section \ref{sec:conclusion}.
 \section{Preliminaries} \label{sec:prelims}

\subsection{Mathematical notation}

We denote by $\R$ and $\R_+$ the sets of real and nonnegative real numbers, respectively.
The $n$-dimensional Euclidean space is denoted by $\R^n$. 
The Euclidean inner product of two vectors $x,y \in \R^n$ is defined as $\langle x, y \rangle = x^\top y$. The associated Euclidean norm of a vector $x \in \R^n$ is $|x| = \sqrt{x^\top x}$.

We denote by $\R^ {m \times n}$ the set of real $m \times n$ matrices. $\mathbb{S}_+(n)$ denotes the set of $n \times n$ symmetric positive definite (SPD) matrices, and the identity matrix is denoted by $I_n \in \R^{n \times n}$.
Given two matrices $A,B \in \R^{m \times n}$, the Euclidean matrix inner product is defined as $\langle A, B\rangle = \mathrm{tr}(A^\top B)$, the Frobenius norm of $A \in \R^{n\times n}$ is given by $\|A\| = \langle A, A\rangle$.
We denote by $A \otimes B$ the Kronecker product of two matrices $A$ and $B$.

The special orthogonal group of 3D rotations is denoted by $\SO(3) := \{R \in \mathbb{R}^{3 \times 3} \mid RR^\top = R^\top R = I_3, \det(R) = 1\}$. 
Its associated Lie algebra, denoted by $\so(3)$, is defined as $\so(3) = \{ \bm{\omega}^\times \mid \bm{\omega} \in \mathbb{R}^{3} \}$, where the map $(\cdot)^\times: \R^3 \rightarrow \so(3)$ denotes the skew-symmetric operator associated with the vector cross-product, i.e. $u^\times v = u \times v$, for all $u, v \in \R^3$.
The exponential map $\exp : \so(3) \rightarrow \SO(3)$ defines a local diffeomorphism from a neighbourhood of $0 \in \so(3)$ to a neighbourhood of $I_3 \in \SO(3)$.
This allows us to define the composition map $\exp \circ (\cdot)^\times: \R^3 \rightarrow \SO(3)$, which can be computed using Rodrigues' formula \cite{ma2004invitation}.

For any $u \in \R^3 \setminus \{0\}$, the projection $\pi_u: \R^3 \setminus \{0\} \rightarrow \R^{3\times 3}$ onto the plane orthogonal to $u$ is given by
$$ \pi_u :=  I_3 - \frac{u u^\top}{|u|^2} = -\frac{u^\times u^\times}{|u|^2}, $$
which satisfies $\pi_u v = 0$ if $u$ and $v$ are collinear.
Extending this definition to all $u \in \R^3$, we define
$
\bar{\pi}_u := |u|^2 I_3 - u u^\top.
$
Unlike the standard projection, this formulation remains well-defined\footnote{It avoids division by $|u|^2$ and ensures continuity when $u$ crosses zero.} when $u = 0$, where it simply reduces to $\bar{\pi}_0 = 0$. For $u \neq 0$, $\bar{\pi}_u$ is a scaled version of the standard projection matrix; that is, $\bar{\pi}_u = |u|^2 \pi_u$. 

The unit sphere $\mathrm{S}^{2} := \{ \eta \in \R^{3} \mid |\eta| = 1 \} \subset \R^{3}$ denotes the set of unit 3D vectors and forms a smooth submanifold of $\R^3$. 
The tangent space at any $\eta \in \S^2$ is $\tT_\eta \S^2 := \{ \gamma \in \R^3 \mid \langle \gamma, \eta \rangle = 0\}$.
Given a smooth function $\calC: \S^2 \rightarrow \R_+$, the gradient of $\calC$ at $\eta \in \S^2$ is defined as the unique tangent vector $\grad \calC(\eta) \in \tT_\eta \S^2$ that satisfies 
$\tD \calC(\eta)[\gamma] = \left\langle \grad \calC(\eta), \gamma \right\rangle, $
for all $\gamma \in \tT_\eta \S^2$, where $\tD \calC(\eta)[\gamma]$ is the differential of $\calC$ along the direction $\gamma$ evaluated at $\eta$, and $\langle \cdot, \cdot \rangle: \tT_\eta \S^2 \times \tT_\eta \S^2 \rightarrow \R$ denotes the Riemannian metric on $\S^2$.
Let $\grad_{\R^3} \calC(\eta)$ denote the Euclidean gradient in $\R^3$, then $\grad \calC(\eta) = \pi_{\eta} \grad_{\R^3} \calC(\eta) \in \tT_{\eta}\S^2$.

A function $g(x)$ is denoted $O(x)$ if it is bounded for any bounded $x \in \R^n$, and $g(x) \rightarrow 0$ as $x \rightarrow 0$.

\subsection{Relevant definitions of observability and stability}
Consider the linear time-varying (LTV) system given by
\begin{equation} \label{eq:ltv_sys}
\begin{cases}
    \dot{\bm{x}} = A(t) \bm{x} + B(t) \bm{u}, \\
    \bm{y} = C(t) \bm{x}.
    \end{cases}
\end{equation}
with state $\bm{x} \in \R^n$, input $\bm{u} \in \R^l$ and output $\bm{y} \in \R^m$. The matrix-valued functions $A(t)$, $B(t)$, and $C(t)$ are assumed to be continuous and bounded.
\begin{definition}[Uniform observability] \label{def_uniformobs}
    The pair $(A(t), C(t))$ 
    is called uniformly observable if there exists $\delta, \mu > 0$ such that $\forall t \geq 0$
    \begin{equation} \label{eq:positivegramian}
        W(t, t+ \delta) := \frac{1}{\delta} \int_t^{t+\delta} \Phi^\top(s,t)C^\top(s)C(s)\Phi(s,t) ds \geq \mu I_n.
    \end{equation}
    where $\Phi(s,t)$ is the transition matrix associated with $A(t)$, i.e. solution to $\frac{d}{dt}\Phi(s,t) = A(t) \Phi(s,t), \forall s \geq t$, $\Phi(t,t) = I_n$.
    \end{definition}
The matrix-valued function $W(t, t + \delta)$ is known as \emph{the observability Gramian} of system \eqref{eq:ltv_sys}. 
A sufficient condition for uniform observability of a subclass of LTV systems is given by the following lemma taken from \cite{hamel2017position}.
\begin{lemma}[\cite{hamel2017position}] \label{lemma:uniform_obs}
Consider System \eqref{eq:ltv_sys} and suppose that  
\begin{itemize} 
\item $C(t) = \Sigma(t) H$ with $H$ a constant matrix, 
   \item $A$ is constant and all its eigenvalues are real,
   \item  the pair $(A, H)$ is Kalman observable, 
\item $\Sigma(t)$ is persistently exciting, i.e., $ \exists \bar{\delta}, \bar{\mu} > 0$ such that, 
$$  \forall t \geq 0, \; M(t, t+\bar{\delta}) := \frac{1}{\bar{\delta}} \int_t^{t+\bar{\delta}} \Sigma(s) ds \geq \bar{\mu} I_n > 0.$$
\end{itemize}
Then, the observability Gramian satisfies the condition \eqref{eq:positivegramian}.
\end{lemma}

\begin{definition}[Almost-global asymptotic stability, \cite{angeli2004almost}]
    An equilibrium point of a nonlinear system is almost-globally asymptotically stable (AGAS) if it is locally stable and its basin of attraction includes all initial conditions with the exception of a set of measure zero.
\end{definition}

\section{Problem Description} \label{sec:problem_desc}
This work considers the visual–inertial odometry problem of estimating the pose and linear velocity of a vehicle equipped with an IMU (gyroscope and accelerometer) rigidly fixed to a monocular camera.

Let $\{\calI\}$ be an inertial frame of reference, and $\{\calB\}$ a body frame attached to the camera-IMU system, which is assumed to coincide with the rigid body’s center of mass. The matrix $R \in \SO(3)$ represents the rotation of frame $\{\calB\}$ with respect to frame $\{\calI\}$. The position and linear velocity of the rigid body, expressed in the inertial frame $\{\calI\}$, are denoted by $\bm{\xi} \in \mathbb{R}^3$ and $\bm{v} \in \mathbb{R}^3$, respectively. 

Suppose the camera provides optical flow data from which the normalized linear velocity vector $\bm{\eta}_v \in \mathrm{S}^2$ in $\{ \calB \}$ is obtained. Then, the considered measurement model is 
\begin{equation} \label{eq:velocity_dir_measurement}   
\bm{\eta}_v = h(R, \bm{v}) := R^\top \frac{\bm{v}}{|\bm{v}|}.
\end{equation}  
In this work, we propose an iterative solution to estimate $\bm{\eta}_v$ from the optical flow of a set of tracked features, which will be presented in Section \ref{sec:opt_flow_alg}. 
The IMU components provide body-fixed measurements of the angular velocity $\bm{\omega} \in \R^3$ and the linear specific acceleration $\bm{a} \in \R^3$.
The second-order kinematics of the vehicle are
\begin{align} \label{eq:dyn_inertial}
    \dot{R} &= R \bm{\omega}^\times, &
    \dot{\bm{\xi}} &= \bm{v}, &
    \dot{\bm{v}} &= R \bm{a} + \bm{g},
\end{align}
 where $\bm{g} = g e_3 \in \R^3$ is the gravitational acceleration in the inertial frame, where $e_3 = (0, 0, 1)^\top \in \mathrm{S}^2$ is the standard gravity direction, and $g \approx 9.81m/s^2$ the gravity constant.

\section{Proposed observer design} \label{sec:observer}
We consider the problem of designing an observer for the attitude $R$, the velocity $\bm{v}$, and position $\bm{\xi}$, using the input measurements $(\bm{\omega}, \bm{a})$, along with the body-frame velocity direction obtained from optical flow measurements and, if available, magnetometer outputs.
Our approach introduces a cascaded state estimation framework that respects the structure of the system and the measurement model in \eqref{eq:velocity_dir_measurement}.

\subsection{Body-frame velocity and gravity direction estimation}
Let $\bm{\xi}^\calB = R^\top \bm{\xi} \in \R^3$ and $\bm{v}^\calB = R^\top \bm{v} \in \R^3$ denote the position and linear velocity of the vehicle expressed in $\{\calB\}$, respectively. Then, their kinematics evolve as
    \begin{align} \label{eq:pos_bodyframe}
            \dot{\bm{\xi}}^\calB &= - \bm{\omega}^\times \bm{\xi}^\calB + \bm{v}^\calB, \\
            \label{eq:vel_bodyframe}
            \dot{\bm{v}}^\calB &= - \bm{\omega}^\times \bm{v}^\calB + \bm{a} + \bm{g}^\calB, 
    \end{align}
    where $\bm{g}^\calB = R^\top \bm{g}$ is the gravity vector expressed in the body frame $\{\calB\}$, which evolves according to 
    \begin{equation} \label{eq:gravity_dyn}
        \dot{\bm{g}}^\calB = - \bm{\omega}^\times \bm{g}^\calB.
    \end{equation}

The estimation of the velocity $\bm{v}^\calB$ and gravity $\bm{g}^\calB$ can be carried out independently of $R$ and $\bm{\xi}^B$. 
To decouple their dynamics from the attitude estimation, we introduce an auxiliary state $\bm{z} \in \R^3$, which serves as an estimate for $\bm{g}^\calB$ in $\{\calB\}$, and propose a Riccati observer, whose stability properties are then established.
In view of the kinematics \eqref{eq:vel_bodyframe} and  \eqref{eq:gravity_dyn}, the observer equations take the form  
    \begin{subequations} \label{eq:riccati_obs}
    \begin{align}
            &\dot{\hat{\bm{v}}}^\calB = - \bm{\omega}^\times \hat{\bm{v}}^\calB + \bm{a} + \bm{z} + \bm{\sigma}_{\bm{v}}, \\
            &\dot{\bm{z}} = - \bm{\omega}^\times \bm{z} + \bm{\sigma}_{\bm{g}},
    \end{align}
    \end{subequations}
    where the correction terms $\bm{\sigma}_v, \bm{\sigma}_z \in \R^3$ are designed to ensure that $(\hat{\bm{v}}^\calB, \bm{z}) \rightarrow (\bm{v}^\calB, \bm{g}^\calB)$.
    Define the error state $ \bm{x} = (\tilde{\bm{v}}^\calB, \tilde{\bm{z}}) = (\bm{v}^\calB - \hat{\bm{v}}^\calB, \bm{g}^\calB - \bm{z})$, it evolves according to
    \begin{align*}
        \ddt \begin{bmatrix}
            \tilde{\bm{v}}^\calB \\ \tilde{\bm{z}}
        \end{bmatrix} &= \begin{bmatrix}
        -\bm{\omega}^\times \tilde{\bm{v}}^\calB + \tilde{\bm{z}} - \bm{\sigma}_{\bm{v}} \\
       -\bm{\omega}^\times \tilde{\bm{z}}  -  \bm{\sigma}_{\bm{g}} 
    \end{bmatrix}, 
    \end{align*} 
    which can be expressed in a compact form as
    \begin{equation} \label{eq:reduced_err_sys}
    \dot{\bm{x}}  = A(t)\bm{x} - \begin{bmatrix}\bm{\sigma}_{\bm{v}} \\ \bm{\sigma}_{\bm{g}} \end{bmatrix}, \;\; A(t) = \begin{bmatrix}
        -\bm{\omega}^\times & I_3 \\ 0_{3,3} & -\bm{\omega}^\times
    \end{bmatrix} \in \R^{6\times 6}. 
    \end{equation} 

    Now, from \eqref{eq:velocity_dir_measurement}, the measured unit velocity vector is defined as $\bm{\eta}_v = \bm{v}^\calB/|\bm{v}^\calB|$, which implies $\pi_{\bm{\eta}_v} \bm{v}^\calB = 0 $. Consequently, one can write 
    $$ - \pi_{\bm{\eta}_v} \hat{\bm{v}}^\calB = \pi_{\bm{\eta}_v} (\bm{v}^\calB - \hat{\bm{v}}^\calB). $$ 
    Defining the output as $\bm{y}(t) := - \pi_{\bm{\eta}_v} \hat{\bm{v}}^\calB$, we obtain the system output equation in standard form \eqref{eq:ltv_sys}
    \begin{align}
        \bm{y}(t) &= C(t) \bm{x}, &
        C(t) &= \begin{bmatrix}
        \pi_{\bm{\eta}_v} & 0_{3,3} 
        \end{bmatrix} \in \R^{3\times 6}. 
    \end{align}
    
    Then, the correction terms are chosen as
    \begin{align} \label{eq:corrections_riccati}
             \bm{\sigma}_{\bm{v}} &= K_{\bm{v}} \bm{y}, &
             \bm{\sigma}_{\bm{g}} &= K_{\bm{g}} \bm{y},
    \end{align}
    where the gain matrix is given by $K = \left[
        K_{\bm{v}}^\top \; K_{\bm{g}}^\top \right]^\top = PC^\top D$, with $P$ solution to the continuous Riccati equation
    \begin{equation} \label{eq:cre}
        \dot{P} = AP + PA^\top - PC^\top D CP + S 
    \end{equation}
    where $S \in \mathbb{S}_+(6)$ and $D \in \mathbb{S}_+(3)$ are bounded continuous symmetric positive definite matrix-valued functions.
    The closed-loop error system evolves as
    \begin{equation}
        \dot{\bm{x}} = (A(t) - K(t) C(t)) \bm{x}.
    \end{equation}  

\begin{remark}
In stochastic Kalman filtering, the matrices $S$ and
$D^{-1}$ involved in the CRE \eqref{eq:cre} are chosen as the covariance matrices of Gaussian noise terms added to the dynamics and output, respectively, to ensure ``optimality'' in the sense of minimizing the estimation error covariance. 
In a noise-free setting, the existence of bounded, well-conditioned solutions to the CRE \eqref{eq:cre} relies on the uniform
observability of the pair $(A, C)$ and choosing $S$ and $D$ as (tunable) SPD matrices. This ensures the existence of an exponentially decreasing Lyapunov function $V(\bm{x}) = \bm{x}^\top P^{-1}(t)\bm{x}$, which guarantees exponential error convergence (see \cite{hamel2017position} for details).
\end{remark}

  \begin{lemma} \label{uniform_obs_lemma}
  Assume that the velocities $\bm{v}^\calB(t)$ and $\bm{\omega}(t)$ are continuous and uniformly bounded. Moreover, 
        assume that the inertial-frame velocity direction $\bm{\eta}_v^\calI := R \bm{\eta}_v \in \S^2$ is persistently exciting in the sense that there exist $\delta, \bar{\mu} > 0$ such that, for all $t \geq 0$,
    \begin{equation} \label{eq:persistent_excitation}
        \frac{1}{\delta} \int_t^{t+\delta} \pi_{\bm{\eta}_v^\calI} d \tau > \bar{\mu} I_3.
     \end{equation}
     Then, the pair $(A(t), C(t))$ is uniformly observable, and the equilibrium $\bm{x} = 0$ is globally exponentially stable (GES). 
    \end{lemma}
    \begin{proof}
    Let us compute the state transition matrix associated with the state matrix $A(t)$. Observe that $A(t) = \bar{A} + S(t)$, where $S(t) = - I_2 \otimes \bm{\omega}^\times$ and 
    $$ \bar{A} = \begin{bmatrix}
        0_{3,3} & I_3 \\ 0_{3,3} & 0_{3,3}
    \end{bmatrix}. $$
    Consider the change of variable $\bar{\bm{x}} = \bar{R}(t) \bm{x}$, where $\bar{R} = I_2 \otimes R(t)$, for any $t \geq 0$. 
    Then by direct differentiation one obtains $\dot{\bar{\bm{x}}} = \bar{A} \bar{\bm{x}}$, 
    which implies the state transition matrix satisfies $\bar{\bm{x}}(s) = \bar{\Phi}(s,t)\bar{\bm{x}}(t)$, with $\bar{\Phi}(s,t) = \exp(\bar{A}(s-t))$.
    Consequently, the original state transition matrix satisfies, $\Phi(s,t) = \bar{R}^\top(s) \bar{\Phi}(s,t) \bar{R}(t)$.
    The observability Gramian of $(A(t), C(t))$ can then be written as $W(t, t+ \delta) = \bar{R}^\top(t) \bar{W}(t, t + \delta) \bar{R}(t)$, 
    where 
    \begin{align*}
        &\bar{W}(t, t + \delta) \\ &= {\displaystyle \frac{1}{\delta} \int_t^{t+\delta} \bar{\Phi}^\top(s,t) \bar{R}(s)C^\top(s)C(s)\bar{R}^\top(s) \bar{\Phi}(s,t) ds }.
    \end{align*}  
    Let $\bar{C}(t) = \bar{R}(t)C(t)\bar{R}^\top(t) = \left[ \pi_{\bm{\eta}_v^\calI} \; 0_{3,3} \right]$, it follows that
    \begin{align*}
        \bar{W}(t, t+ \delta) 
        &=  \frac{1}{\delta} \int_t^{t+\delta} \bar{\Phi}^\top(s,t) \bar{C}^\top(s)\bar{C}(s)\bar{\Phi}(s,t) ds.
    \end{align*}
    Thus, $\bar{W}(t, t+\delta)$ is the observability Gramian associated with the pair $(\bar{A}, \bar{C}(t))$. Notice that $\bar{C}(t)$ admits the factorization $\bar{C}(t) = \Sigma(t) H$, where $H := \left[I_3 \; 0_{3,3}  \right]$ and $\Sigma(t) := \pi_{\bm{\eta}_v^\calI}$.
    Since $\rank\left(\left[\begin{smallmatrix}
        H \\ H \bar{A}
    \end{smallmatrix}\right]\right) = 6$, the pair $(\bar{A}, H)$ is Kalman observable.
    Moreover, the persistent excitation condition \eqref{eq:persistent_excitation} ensures that $\frac{1}{\delta} \int_t^{t+\delta} \Sigma(s) ds \geq \bar{\mu}I_3$ for all $t \geq 0$. Applying Lemma \ref{lemma:uniform_obs}, we conclude that $\bar{W}(t, t + \delta) \geq \bar{\mu} I_3$, for any $t \geq 0$. Thus, the pair $(\bar{A}, \bar{C}(t))$, and by extension the pair $(A(t), C(t))$, is uniformly observable.
    This in turn guarantees the global exponential stability (GES) of system \eqref{eq:reduced_err_sys} at the equilibrium $(\tilde{\bm{v}}^\calB, \tilde{\bm{z}}) = (0,0)$.
    \end{proof}

Lemma \ref{uniform_obs_lemma} establishes a persistent excitation condition on the inertial-frame velocity direction $\bm{\eta}_v^\calI$, which requires the vehicle's motion direction to change persistently over time. 

\begin{figure}[t]
    \centering
    \includegraphics[width=.95\linewidth]{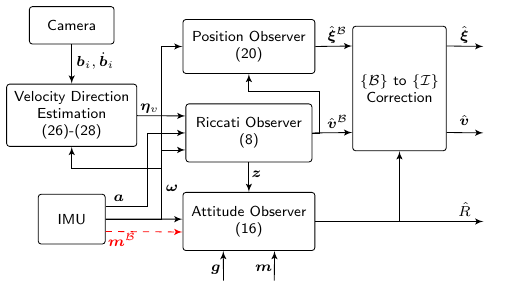}
    \caption{Diagram of the proposed optical flow-based VIO architecture with cascaded observer design.
    The velocity direction $\bm{\eta}_v$ is estimated from optical flow and, together with IMU data, is provided as input to a Riccati observer to estimate $\hat{\bm{v}}^\calB$ and $\bm{z}$. The estimate $\bm{z}$ is used for attitude reconstruction, while $\hat{\bm{v}}^\calB$ is integrated for position. The rotation estimate $\hat{R}$ transforms all estimates to the inertial frame ${ \calI }$. The magnetometer (in red) is optional and used to estimate yaw in the attitude observer.} 
    \label{fig:cascade_obs}
\end{figure}

\subsection{Attitude estimation on $\SO(3)$}
In this second stage, we use the estimated body-frame gravity $\bm{z}$ to design a complementary filter on $\SO(3)$ to estimate the attitude $R$ (up to a yaw rotation about $\bm{g}$) \cite{mahony2008nonlinear}, given that $\bm{g}$ is known and constant. 
To estimate the full attitude, an additional direction measurement is needed. While many sensors can provide such information, we assume that the IMU includes a magnetometer. The measured Earth’s magnetic field $\bm{m}^\calB \in \mathrm{S}^2$ is then modelled as
\begin{align} \label{eq:mag_model}
\bm{m}^\calB = R^\top \bm{m}, 
\end{align}
where $\bm{m} \in \mathrm{S}^2$ is the reference magnetic field direction.
Although this measurement is not required for pitch and roll estimation, we include it to illustrate how an additional known direction can aid in reconstructing the full attitude.

Let $\hat{R} \in \SO(3)$ denote the estimate of the attitude $R$. The observer is given by
\begin{equation} \label{eq:observer_R}
            \dot{\hat{R}} = \hat{R} \bm{\omega}^{\times} - \bm{\sigma}_R^\times \hat{R}, \quad \hat{R}(0) \in \SO(3),
    \end{equation}
    where the correction term $\bm{\sigma}_R \in \R^3$ is defined as 
    \begin{equation} \label{eq:innovation_R}
        \bm{\sigma}_R = k_z (\bm{g}^\times \hat{R} \bm{z}) + k_m (\bar{\bm{m}}^\times \hat{R} \bar{\bm{m}}^{\calB}).
    \end{equation}
    with $\bar{\bm{m}} = \pi_{\bm{g}}\bm{m}$, $\bar{\bm{m}}^\calB = \bar{\pi}_{\bm{z}}\bm{m}^\calB$, $k_z >0$, $k_m \geq 0$, where $\bar{\pi}_{\bm{z}}$ is used instead of the projection $\pi_{\bm{z}}$ to account for cases where $\bm{z}$ may vanish. The use of $\bar{\bm{m}}$ and $\bar{\bm{m}}^\calB$ instead of $\bm{m}$ and $\bm{m}^\calB$ ensures that the yaw estimation is decoupled from roll and pitch.    Define the error $\tilde{R} = \hat{R} R^\top \in \SO(3)$. Differentiating $\tilde{R}$, we obtain
        \begin{align} 
        \dot{\tilde{R}} &=  (\hat{R}\bm{\omega}^\times - \bm{\sigma}_R^\times \hat{R}) R^\top - \hat{R} \bm{\omega^\times} R^\top, \notag \\ \label{eq:R_err_dyn}
        &= - \left( k_z \bm{g}^\times \hat{R} \bm{z} + k_m \bar{\bm{m}}^\times \hat{R} \bar{\bm{m}}^{\calB}  \right)^\times \tilde{R}, 
        `\end{align}%
\begin{theorem}
Consider the Riccati observer \eqref{eq:riccati_obs} with corrections \eqref{eq:corrections_riccati} and the attitude observer \eqref{eq:observer_R} with correction \eqref{eq:innovation_R}.
If the pair $(A(t), C(t))$ is uniformly observable according to Lemma \ref{uniform_obs_lemma}.
    Then: 
    \begin{itemize}
        \item[(1)] If $k_m = 0$ (without magnetometer), 
        the errors $(\tilde{R}\bm{g}, \bm{x})$ converge to the set of equilibria $\mathcal{E}^g = (\bm{g},0) \cup (-\bm{g},0)$. Moreover, the equilibrium $(-\bm{g},0)$ is unstable, and the equilibrium $(\bm{g},0)$ is almost-globally asymptotically stable.

        \item[(2)] If $k_m > 0$ (magnetometer available), the errors $(\tilde{R}, \bm{x})$ converge to set of equilibria $\mathcal{E} = \mathcal{E}_s \cup \mathcal{E}_u$, where $\mathcal{E}_s = \{(I_3, 0)\}$ and
        $$ \mathcal{E}_u = \{(U \Lambda U^\top,0) \mid \Lambda = \diag(1, -1, -1), U \in \SO(3)\}. $$
        Moreover, the set of equilibria $\mathcal{E}_u$ is unstable, and the singleton set $\mathcal{E}_s$ is almost-globally asymptotically stable.
    \end{itemize}
\end{theorem}

    \begin{proof}
    Relation \eqref{eq:R_err_dyn} can be rewritten in terms of $\tilde{\bm{z}}$ as 
\begin{align*}
        \dot{\tilde{R}} &= - \left( k_z \bm{g}^\times \hat{R} \bm{g}^\calB - k_z \bm{g}^\times \hat{R} \tilde{\bm{z}} + k_m \bar{\bm{m}}^\times \hat{R} \bar{\bm{m}}^{\calB} \right)^\times \tilde{R}. %
    \end{align*}
    From Lemma \ref{uniform_obs_lemma}, it follows that $\bm{z} \rightarrow R^\top \bm{g}$ exponentially, which implies that $\bar{\bm{m}}^\calB \rightarrow \bar{\bar{\bm{m}}}^\calB$, with $\bar{\bar{\bm{m}}}^\calB = \bar{\pi}_{\bm{g}^\calB}\bm{m}^\calB$.
    Moreover, one can show that $\bar{\bm{m}}^\calB = \bar{\bar{\bm{m}}}^\calB + O(\tilde{\bm{z}})$.
Expressing this in terms of $\bm{x}$, one obtains:
\begin{align*}
    \dot{\tilde{R}} &= - \left( k_z \bm{g}^\times \hat{R} \bm{g}^\calB + k_m \bar{\bm{m}}^\times \hat{R} \bar{\bar{\bm{m}}}^{\calB} + O(\bm{x}) \right)^\times \tilde{R}.      
         \end{align*}
    This yields the following closed-loop system
\begin{subequations} \label{eq:interconnection}
\begin{align} \label{eq:interconnection1}
\dot{\tilde{R}} &= - \left( k_z \bm{g}^\times \hat{R} \bm{g}^\calB + k_m \bar{\bm{m}}^\times \hat{R} \bar{\bar{\bm{m}}}^{\calB} + O(\bm{x}) \right)^\times \tilde{R},  \\
    \label{eq:interconnection2}
    \dot{\bm{x}} &= (A -KC) \bm{x}.
\end{align}
\end{subequations}
The above system can be seen as a cascade interconnection of a nonlinear system on $\SO(3)$ \eqref{eq:interconnection1} and the LTV system on $\R^6$ \eqref{eq:interconnection2}.
To prove the almost global stability of the interconnection system, we begin by proving that subsystem \eqref{eq:interconnection1} for $\bm{x} = 0$ is AGAS.

Proof of item 1) 
By examining subsystem \eqref{eq:interconnection1} for $\bm{x} = 0$ and $k_m = 0$, one has
$ \dot{\tilde{R}} = - ( k_z \bm{g}^\times \hat{R} \bm{g}^\calB  )^\times \tilde{R}$.
Define the reduced attitude $\tilde{\bm{g}} = \tilde{R} \bm{g}$. Then, differentiating both sides yields $\dot{\tilde{\bm{g}}} = - k_z(\bm{g}^\times \tilde{\bm{g}})^\times \tilde{\bm{g}}$.
From this, one can establish that the equilibrium points are $\tilde{\bm{g}} = \bm{g}$ and $\tilde{\bm{g}} = -\bm{g}$, which correspond to the sets $\mathcal{E}^g_s = \{\tilde{R} \in \SO(3) \mid \tilde{R}\bm{g} = \bm{g}\}$ and $\mathcal{E}^g_u = \{\tilde{R} \in \SO(3) \mid \tilde{R}\bm{g} = -\bm{g} \}$, respectively. 
Consider the candidate Lyapunov function $\mathcal{L}(\tilde{\bm{g}}) = |\bm{g}|^2 - \bm{g}^\top \tilde{\bm{g}}$, differentiating it yields $\dot{\mathcal{L}}(\tilde{\bm{g}}) = - k_z \bm{g}^\top(\bm{g}^\times \tilde{\bm{g}})^\times \tilde{\bm{g}}$, using the identity $u^\top (u^\times v)^\times v = -|u^\times v|^2$ for any $u,v \in \R^3$, we get $\dot{\mathcal{L}} = -k_z|\bm{g}^\times \tilde{\bm{g}}|^2 \leq 0$. Since $\tilde{\bm{g}} = \bm{g}$ is the only global minimum of $\mathcal{L}$, it follows that $\mathcal{E}^g_s$ is asymptotically stable. On the other hand, we have $\mathcal{L}(-\bm{g}) = 2|\bm{g}|^2$, then $\tilde{\bm{g}} = -\bm{g}$ is a maximum of $\mathcal{L}$. Moreover, one shows that $\mathcal{L}$ locally decreases for any small perturbation; this implies that the equilibrium $\mathcal{E}^g_u$ is unstable. 
The set $\mathcal{E}^g_u$ forms a single point on the 2-dimensional scaled sphere  $\S^2_{\bm{g}} := \{ u \in \R^3 \mid u^\top u = |\bm{g}|^2 \}$ while its complement is open and dense in $\S^2_{\bm{g}}$. Therefore, $\mathcal{E}^g_u$ has measure zero in $\S^2_{\bm{g}}$.
It follows that the stable equilibrium $\tilde{\bm{g}} = \bm{g}$ is AGAS. 

To complete the proof, we now examine the full interconnection system.  
Since $\eqref{eq:interconnection2}$ evolves independently of $\tilde{R}$ and is GES from Lemma \ref{uniform_obs_lemma}, there exist $c, \alpha > 0$ such that $\bm{x}$ satisfies  $|\bm{x}(t)| \leq c \exp(-\alpha t) |\bm{x}(0)|$, for all $t \geq 0$.  
Thus, $\bm{x}$ remains uniformly bounded, meaning there exists a compact set $\calS \subset \R^6$ such that $\bm{x}(t) \in \calS$ for all $t \geq 0$. 
Therefore, according to \cite[Proposition 2]{angeli2010stability}, one can conclude that subsystem \eqref{eq:interconnection1} is almost globally ISS with respect to $\tilde{R}\bm{g} = \bm{g}$ and input $\bm{x}$.
Hence, given that $\bm{x} = 0$ for system \eqref{eq:interconnection2} is GES and subsystem \eqref{eq:interconnection1} with $\bm{x} = 0$ is AGAS at $\tilde{R} \bm{g} = \bm{g}$ and almost globally ISS with respect to $\bm{x}$, it follows from \cite[Theorem 2]{angeli2004almost} that the cascaded interconnection system \eqref{eq:interconnection} is AGAS at $ (\tilde{R}\bm{g}, \bm{x})= (\bm{g}, 0)$.  

Proof of item 2) In this case, the attitude dynamics satisfy $\dot{\tilde{R}} = - ( k_z \bm{g}^\times \hat{R} \bm{g}^\calB + k_m \bar{\bm{m}}^\times \hat{R} \bar{\bar{\bm{m}}}^{\calB} )^\times \tilde{R}$. From there, it can be shown that the equilibrium sets are  $\mathcal{E}_s = \{I_3\}$ and $\mathcal{E}_u = \{ \tilde{R} = U \Lambda U^\top \mid \Lambda = \diag(1, -1, -1), U \in \SO(3)\}$ (see \cite{mahony2008nonlinear}). 
The singleton set $\mathcal{E}_s$ is the stable equilibrium, and the set $\mathcal{E}_u$ is the set of unstable equilibria \cite[Theorem 6.1]{van2023synchronous}, it consists of all 180-degree rotations, each defined by an axis on $\S^2$, which corresponds to a 2-dimensional space embedded in the three-dimensional manifold $\SO(3)$, and thus has measure zero in $\SO(3)$.
It follows that the stable equilibrium $\tilde{R} = I_3$ is almost globally asymptotically stable for subsystem \eqref{eq:interconnection1}.
The remainder of the proof proceeds similarly to that of item (1) to conclude the almost global stability of $(\tilde{R}, \bm{x}) = (I_3, 0)$.
This completes the proof.
\end{proof}

\begin{remark}
The body-frame position estimate $\hat{\bm{\xi}}^\calB$ follows directly from integrating the dynamics \eqref{eq:pos_bodyframe} using the velocity estimate,
\begin{equation}
    \dot{\hat{\bm{\xi}}}^\calB = - \bm{\omega}^\times \hat{\bm{\xi}}^\calB + \hat{\bm{v}}^\calB.
\end{equation} 
Since $\tilde{\bm{v}}^\calB \rightarrow 0$ and $\dot{\tilde R} \rightarrow 0$, it is not difficult to show that the position error $\tilde{\bm{\xi}} := \bm{\xi} - \hat R\hat{\bm{\xi}}^\calB$ satisfies $\dot{\tilde{\bm{\xi}}} \rightarrow 0$ and hence $\tilde{\bm{\xi}} \rightarrow \bm{\xi}^\star\in\mathbb{R}^3$. Therefore, since $\tilde R \rightarrow R^\star\in\mathbf{SO}(3)$, one has $\hat{\bm{\xi}}:=\hat R\hat{\bm{\xi}}^\calB \rightarrow (R^\star)^\top (\bm{\xi} - \bm{\xi}^\star)$. The position is estimated up to an unknown inertial position and an unknown orientation about $e_3$ (since $R^\star e_3=e_3$). If a magnetometer is used, the position is estimated up to an unknown inertial position only.

\end{remark}

\begin{remark}
The proposed estimation scheme can equivalently be formulated in the inertial frame. 
Define the inertial-frame estimate of the gravity vector as 
$\bm{z}^\calI := \hat{R}\bm{z}$, and recall that 
$\hat{\bm{\xi}} = \hat{R}\hat{\bm{\xi}}^\calB$ and 
$\hat{\bm{v}} = \hat{R}\hat{\bm{v}}^\calB$. 
The overall estimation dynamics can then be expressed on $SO(3)\times\R^9$ as
\begin{subequations}
\begin{align}
\dot{\hat{R}} &= \hat{R}\bm{\omega}^\times - \bm{\sigma}_R^\times \hat{R}, \\
\dot{\hat{\bm{\xi}}} &= \hat{\bm{v}} - \bm{\sigma}_R^\times \hat{\bm{\xi}}, \\
\dot{\hat{\bm{v}}} &= \bm{z}^\calI + \hat{R}\bm{a} - \bm{\sigma}_R^\times \hat{\bm{v}} + \hat{R}\bm{\sigma}_{\bm{v}}, \\
\dot{\bm{z}}^\calI &= -\bm{\sigma}_R^\times \bm{z}^\calI + \hat{R}\bm{\sigma}_{\bm{g}}.
\end{align}
\end{subequations}
Here, the correction terms $\bm{\sigma}_{\bm{v}}$ and $\bm{\sigma}_{\bm{g}}$ are given by 
\eqref{eq:corrections_riccati}, and $\bm{\sigma}_R$ is defined in \eqref{eq:innovation_R}. Unlike the body-frame formulation in Section~\ref{sec:observer}, which yields a linear Riccati 
observer decoupled from the attitude dynamics, the inertial-frame formulation results in a 
nonlinear observer where orientation and translational motion are intrinsically coupled 
through $\bm{\sigma}_R$. This explains why the body-frame formulation was preferred for the 
stability analysis. Nevertheless, the inertial-frame version can be appealing in practice, 
since all states are expressed directly in the inertial frame and no frame transformation 
from body to inertial coordinates is required at each integration step.
\end{remark}

\section{Velocity direction computation from sparse optical flow measurements} \label{sec:opt_flow_alg}
Estimating the unit velocity vector $\bm{\eta}_v$ based on optical flow measurements is critical to implementing the observer presented in section \ref{sec:observer}. 
This section proposes a gradient descent algorithm that exploits optical flow data of a set of sparse image feature points to estimate $\bm{\eta}_v$.

\begin{assumption} \label{assump:stationay_points}
    The camera observes a set of $n$ landmarks $Y_1, \dots, Y_n$, that are stationary in the inertial frame. In addition, its intrinsic parameters are known.
\end{assumption}

\subsection{Optical flow representation on spherical images}

We represent optical flow on a spherical image \cite{mahony2008dynamic} by expressing each landmark projection as a bearing on the unit sphere, given by $\bm{b}_i := Y_i /|Y_i| \in \S^2$, for $i=1, \dots,n$.
This representation does not require a physical spherical camera; it can be obtained by projecting calibrated perspective coordinates onto a unit sphere image surface \cite{ma2004invitation}.

Since the bearing vectors inherit the camera's rigid motion under Assumption \ref{assump:stationay_points}, 
the time derivative $\dot{\bm{b}}_i$ (i.e., the spherical optical flow) of the $i$-th bearing $\bm{b}_i$ is given by (see \cite{mahony2008dynamic}):
\begin{equation} \label{eq:optflow_eq}
    \dot{\bm{b}}_i =  \frac{1}{|Y_i|} \pi_{\bm{b}_i} \bm{v}^\calB - \bm{\omega}^\times \bm{b}_i
\end{equation} 
where $|Y_i|$ represents the range associated with the $i$-th bearing.
Classical optical flow estimation algorithms, such as Lucas-Kanade \cite{lucas1981iterative} or Horn-Schunck \cite{horn1981determining}, can be used to estimate the motion of features across consecutive frames.

\subsection{Proposed algorithm}

The angular velocity $\bm{\omega}$ in equation \eqref{eq:optflow_eq} can be provided by the gyroscope and is thus assumed to be known. Consequently, one can subtract the rotational components from the measured flow, such that equation \eqref{eq:optflow_eq} can be rewritten as
    $\bm{s}_i = \alpha_i \pi_{\bm{b}_i} \bm{\eta}_v$, 
where $\bm{s}_i = \dot{\bm{b}}_i + \bm{\omega}^\times \bm{b}_i$ is the (measured) translational optical flow, and $\alpha_i = |\bm{v}|/|Y_i|>0$, $i=1, \dots, n$, are unknown scale factors.

Since $|\bm{s}_i| = \alpha_i |\pi_{\bm{b}_i} \bm{\eta}_v|$, it follows that, for all $i=1,\dots,n$, the vector $\bm{\eta}_v$ satisfies the following constraint
\begin{equation} \label{eq:constraint_si}
    \frac{\bm{s}_i}{|\bm{s}_i|} = \frac{\pi_{\bm{b}_i} \bm{\eta}_v}{|\pi_{\bm{b}_i} \bm{\eta}_v|},
\end{equation} 
or equivalently,
$
    |\pi_{\bm{b}_i} \bm{\eta}_v|\bm{s}_i - |\bm{s}_i| \pi_{\bm{b}_i} \bm{\eta}_v = 0. 
$ 
Let $\hat{\bm{\eta}}_v \in \S^2$ be an estimate of $\bm{\eta}_v$. Then, estimating $\bm{\eta}_v$ reduces to solving the following constrained minimization problem:
\begin{equation} \label{eq:eta_min}
\hat{\bm{\eta}}_v = \arg \min_{\hat{\bm{\eta}}_v \in \S^2} \calC(\hat{\bm{\eta}}_v), 
\end{equation}
where the cost function $\calC: \S^2 \rightarrow \R_+$ is given by
\begin{align}
    \calC(\hat{\bm{\eta}}_v) &= \frac{1}{2}\sum_{i=1}^n  \left| |\pi_{\bm{b}_i} \hat{\bm{\eta}}_v|\bm{s}_i - |\bm{s}_i| \pi_{\bm{b}_i} \hat{\bm{\eta}}_v \right|^2, \notag \\
    &= \sum_{i=1}^n \left( |\bm{s}_i|^2 \hat{\bm{\eta}}_v^\top \pi_{\bm{b}_i} \hat{\bm{\eta}}_v - |\bm{s}_i| |\pi_{\bm{b}_i} \hat{\bm{\eta}}_v| \bm{s}_i^\top \pi_{\bm{b}_i} \hat{\bm{\eta}}_v \right).
    \end{align}

To solve the problem in \eqref{eq:eta_min}, we apply a gradient descent algorithm on the manifold $\S^2$, 
where at each iteration, the search direction is constrained to the tangent space $\tT_{\hat{\bm{\eta}}_v} \S^2$.
Define $\bm{\delta} \in \R^3$, such that $\bm{\delta}^\times \hat{\bm{\eta}}_v \in \tT_{\hat{\bm{\eta}}_v} \S^2$ describes the steepest descent direction. Computing the differential of $\calC(\hat{\bm{\eta}}_v)$ along the tangent direction $\bm{\delta}^\times \hat{\bm{\eta}}_v$ yields
 \begin{align*}
     \tD \calC(\hat{\bm{\eta}}_v)[\bm{\delta}^\times \hat{\bm{\eta}}_v] &= - \sum_{i=1}^n \biggl( 2 |\bm{s}_i|^2 \hat{\bm{\eta}}_v^\top \pi_{\bm{b}_i} 
     - |\bm{s}_i||\pi_{\bm{b}_i}\hat{\bm{\eta}}_v| \bm{s}_i^\top \\ & \qquad \qquad  - \frac{|\bm{s}_i|}{|\pi_{\bm{b}_i}\hat{\bm{\eta}}_v|} \bm{s}_i^\top \pi_{\bm{b}_i} \hat{\bm{\eta}}_v \hat{\bm{\eta}}_v^\top \pi_{\bm{b}_i}  \biggr) \hat{\bm{\eta}}_v^\times \bm{\delta}, \\
     &=  - \left\langle \grad\calC(\hat{\bm{\eta}}_v),   \hat{\bm{\eta}}_v^\times \bm{\delta} \right\rangle, \\
     &=  \left\langle \hat{\bm{\eta}}_v^\times \grad\calC(\hat{\bm{\eta}}_v),  \bm{\delta} \right\rangle.
 \end{align*}
Then, a gradient descent algorithm to minimize $\calC$ is given by the following iterative update:
\begin{align} \label{eq:update_etav}
    \hat{\bm{\eta}}_{v, k+1} &= \sign\left(\bm{\beta}_k \right)\exp( \bm{\delta}_k^{\times} ) \hat{\bm{\eta}}_{v, k}, \\
    \bm{\delta}_{k} &= - \kappa \hat{\bm{\eta}}_{v,k}^\times \grad \calC(\hat{\bm{\eta}}_{v,k}), \\
    \bm{\beta}_k &= \frac{1}{n} \sum_{i=1}^n  \frac{\bm{s}_i^\top \pi_{\bm{b}_i} \hat{\bm{\eta}}_{v, k}}{|\bm{s}_i| |\pi_{\bm{b}_i} \hat{\bm{\eta}}_{v, k}|}.
    \end{align} 
with step size $\kappa > 0$.

\begin{remark}
    The term $\sign\left(\bm{\beta}_k \right) \in \{+1, -1\}$, where $\bm{\beta}_k$ is derived using the constraints in \eqref{eq:constraint_si}, determines the appropriate hemisphere of $\S^2$ to which $\hat{\bm{\eta}}_v$ belongs. It accounts for direction flips along an axis, as the constraint $\hat{\bm{\eta}}_v \in \S^2$ introduces an ambiguity when $\bm{v}^\calB$ crosses zero. 
\end{remark}

\begin{remark}
In practice, outlier rejection is essential to ensure accurate point features and their flow. 
When outliers are prevalent, robust estimation methods like M-estimators \cite{stewart1999robust} can be introduced into the algorithm to reduce their effect. To illustrate, a robust cost function 
$\calC_r: \S^2 \rightarrow \R_+$ can be defined as:
$\calC_r(\hat{\bm{\eta}}_v) = \sum_{i=1}^n  \bm{\rho}\left( |\pi_{\bm{b}_i} \hat{\bm{\eta}}_v|\bm{s}_i - |\bm{s}_i| \pi_{\bm{b}_i} \hat{\bm{\eta}}_v \right)$, 
where $\bm{\rho}(\cdot)$ is a positive-definite function with a unique minimum at zero. Various functions are proposed in the literature, such as the Tukey and Huber functions \cite{stewart1999robust}.
\end{remark}

\section{Simulation results}\label{sec:simresults}
In this section, we provide simulation results to evaluate the performance of the observer presented in Section \ref{sec:observer}, along with the iterative algorithm proposed in Section \ref{sec:opt_flow_alg}. 

\subsection{Iterative algorithm simulation}
To evaluate the performance of the iterative algorithm in Section \ref{sec:opt_flow_alg}, we consider a simulation scenario where a vehicle, equipped with a monocular camera and a gyro, observes eight landmarks positioned in the inertial frame at: $ Y_1^\calI = [1,0,0]^\top$, $Y_2^\calI = [-1,0,0]^\top$, $Y_3^\calI = [0,2,0]^\top$, $Y_4^\calI = [0,-2,0]^\top, Y_5^\calI = [1,1,-0.5]^\top, Y_6^\calI = [-1,1,-0.5]^\top, Y_7^\calI = [2,1,1]^\top, Y_8^\calI = [0,-2,-1]^\top$, and follows a predefined trajectory with linear velocity: 
\begin{align*}
    \bm{v}^\calB(t) &= [0.4 \cos(0.2t), -0.4 \sin(0.4t), -0.5 \sin(t)]^\top, 
    \end{align*}
    and measured angular velocity:
    \begin{align*}
    \bm{\omega}(t) &= [0.1 \cos(0.1t); 0.05 \sin(0.15t); 0.05  \sin(0.1 t)]^\top.
\end{align*} 

Measurements are provided at a rate of $50$Hz. At each time step, the bearing vectors are computed as the projective spherical coordinates of the landmarks, and are corrupted by zero-mean Gaussian noise with a standard deviation of $0.1$. 
The algorithm performs $N$ gradient descent iterations per time step using optical flow measurements from a subset of the landmarks computed using  \eqref{eq:optflow_eq}, with initial estimate set to $\hat{\bm{\eta}}_v(0) = e_3$ and step size $\kappa = 5$.
\begin{figure}[h]
    \centering
    \includegraphics[width=1.05\linewidth]{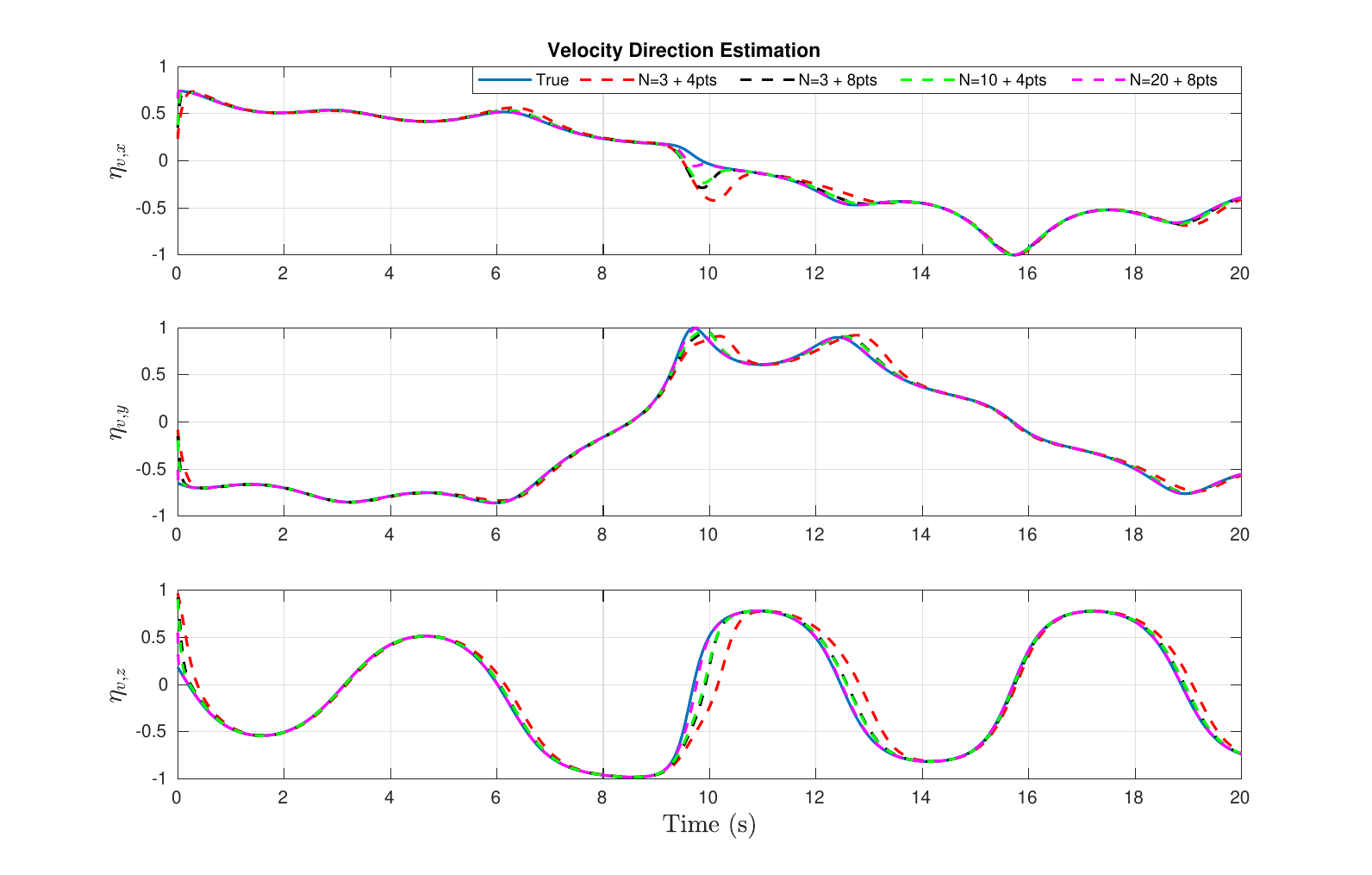}
    \vspace{-.4cm}
    \caption{Estimated velocity direction with $N=3$ iterations per time step using $4$ optical flow measurements (red) and $8$ measurements (black), $N=10$ with $4$ measurements (green), and $N=20$ with $8$ measurements (magenta), compared to the true velocity direction (blue).}
    \label{fig:etav_alg_sim}
\end{figure}
\vspace{-.2cm}
\begin{figure}[h]
    \centering
    \includegraphics[width=1.05\linewidth]{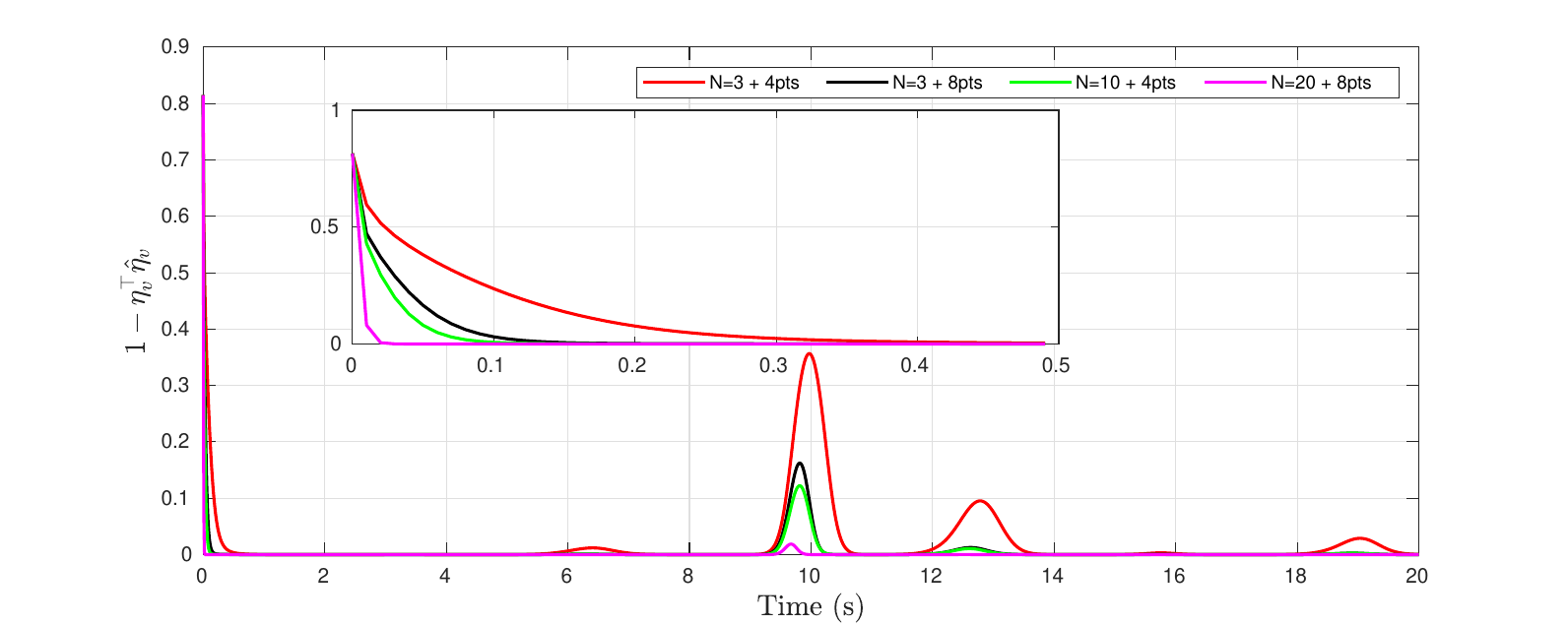}
    \vspace{-.4cm}
    \caption{Estimation errors $1-\hat{\bm{\eta}}^\top \bm{\eta}_v$ for $N=3$ iterations per time step using $4$ optical flow measurements (red) and $8$ measurements (black), $N=10$ with $4$ measurements (green), and $N=20$ with 8 measurements (magenta).}
    \label{fig:etav_alg_err}
\end{figure}

Figure \ref{fig:etav_alg_sim} show the estimated velocity direction components obtained through the iterative gradient descent algorithm for different settings: $N = 3$ with $4$ optical flow measurements, $N = 3$ with $8$, $N= 10$ with $4$, and $N = 20$ with $8$, compared to the true velocity direction $\bm{\eta}_v = \bm{v}^\calB/|\bm{v}^\calB|$.
Figure \ref{fig:etav_alg_err} presents the corresponding estimation errors given by $1 - \bm{\eta}_v^\top \hat{\bm{\eta}}_v$.
One can easily observe that the proposed algorithm demonstrates strong performance and that its accuracy improves as $N$ or the number of optical flow measurements increases. In practice, choosing $N$ and the number of measurements involves a trade-off between the desired accuracy and the algorithm's computational cost.

\subsection{Observer simulation}
To evaluate the performance of the cascaded observer, we simulated a vehicle equipped with a camera-IMU system moving in 3D space prescribed with angular velocity and linear acceleration:
\begin{align*}
    \bm{\omega} &= [0.15 \sin(0.8 t + \pi), 0.1 \sin(t), 0.05 \sin(0.1t + \pi/3)], \\
    \bm{a} &= R^\top [-5 \cos(3t), -5\sin(4t), 9.81 + 5 \sin(4t)].
\end{align*}

A Monte Carlo simulation with $30$ runs is performed, with initial estimates randomly sampled from Gaussian distributions.
The initial estimates for the body-frame velocity and gravity vectors are normally distributed around $\hat{\bm{v}}^\calB(0) = [0.8, -1.5, 0.5]^\top$ and $\bm{z}(0) = [1.4, 0.8, -8.81]^\top$, with standard deviations of $3$ and $1.5$ per axis, respectively.  The initial orientation estimates are distributed normally around $\hat{R}(0) = I_3$ with a standard deviation of $15^\circ$ per axis.
The parameters set as: $S = 30 I_6$, $D = 5 \pi_{\bm{\eta}_v}$, $k_z = 1$, $k_m = 0.1$. The constant vector $\bm{m}$ is set to $\bm{m} = [0, 1/\sqrt{2}, 1/\sqrt{2}]^\top$. 

The velocity direction is estimated using the proposed iterative algorithm based on the optical flow measurements of eight landmarks with $N = 20$ iterations.
The IMU measurements $\bm{\omega}$, $\bm{a}$ and $\bm{m}^\calB$ are corrupted with zero-mean Gaussian noise with standard deviations $0.05$, $0.05$ and $0.2$, respectively. 

Figures \ref{fig:vB_estimation} and \ref{fig:gB_estimation} 
illustrate the estimated and true body-frame linear velocity and gravity vectors, along with their respective estimation errors $|\bm{v}^\calB - \hat{\bm{v}}^\calB|$ and $| \bm{g}^\calB - \bm{z}|$. The shaded areas represent the 5th to 95th percentile.
Figure \ref{fig:R_estimation} depicts the full attitude estimation error $\|I_3 - \hat{R} R^\top\|$ with magnetometer, and the reduced attitude (vertical direction) error $1- \hat{\bm{g}}^\top \bm{g}/|\bm{g}|^2$
where $\hat{\bm{g}} = \hat{R} \bm{g}^\calB$. 
One observes that all estimation errors rapidly decrease and converge to zero for all initial conditions. These results confirm the expected asymptotic convergence of the proposed observer.
\begin{figure}[!h]
    \centering
    \includegraphics[width=1.03\linewidth]{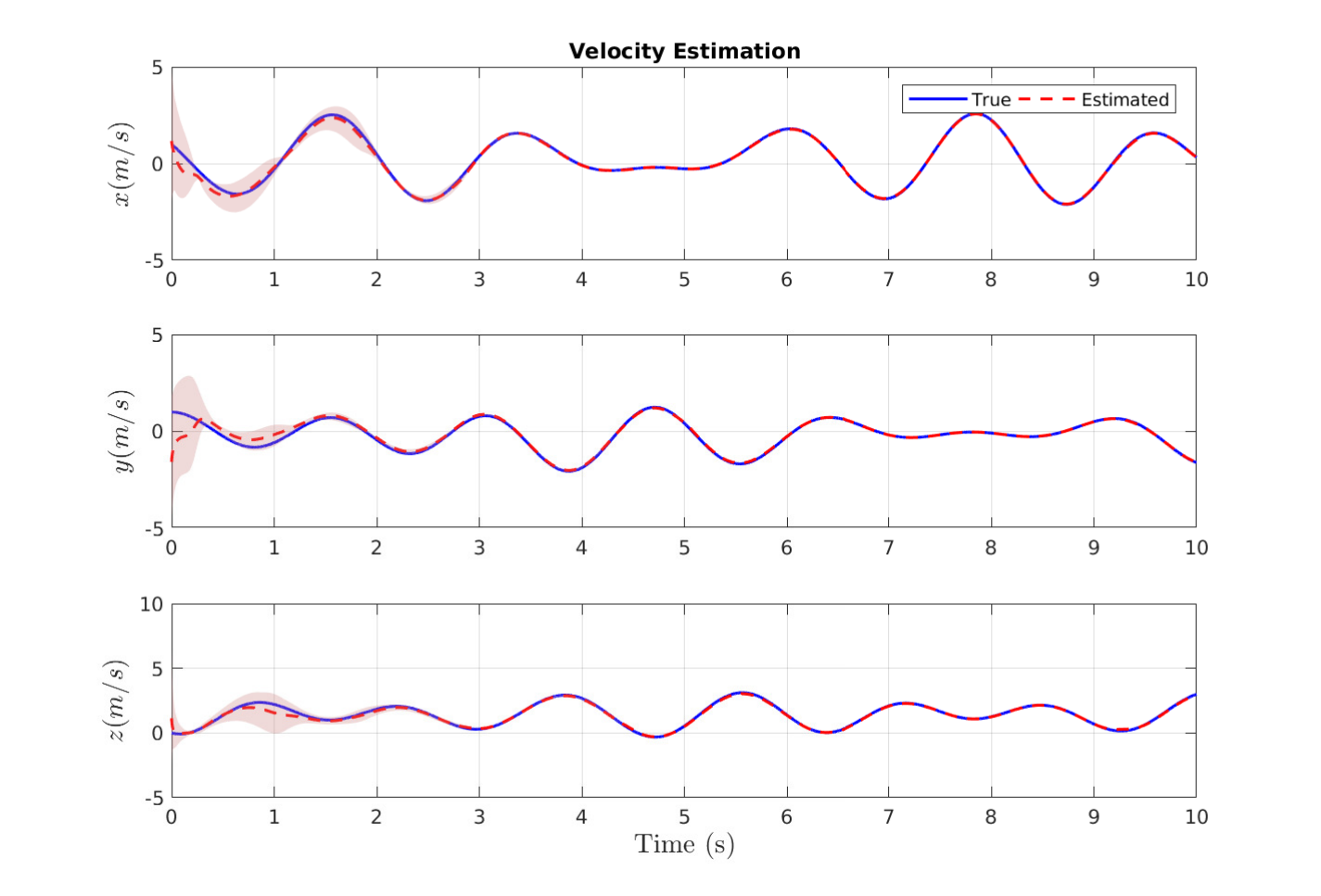}
    \includegraphics[width=1.03\linewidth]{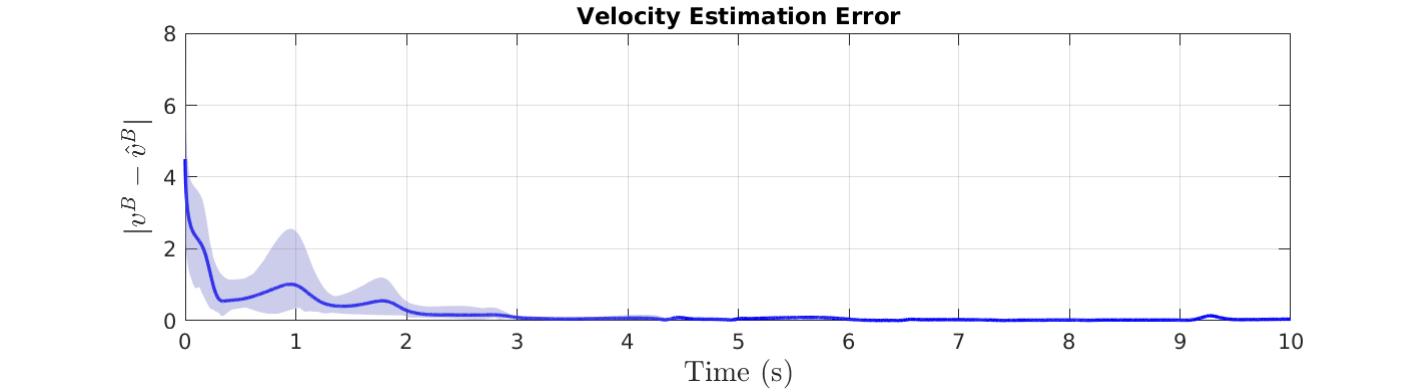}
    \vspace{-.3cm}
    \caption{Estimated and true body-frame linear velocity components, and observer estimation error $|\bm{v}^\calB - \hat{\bm{v}}^\calB|$.}
    \label{fig:vB_estimation}
\end{figure}

\begin{figure}[h]
    \centering
    \includegraphics[width=1.03\linewidth]{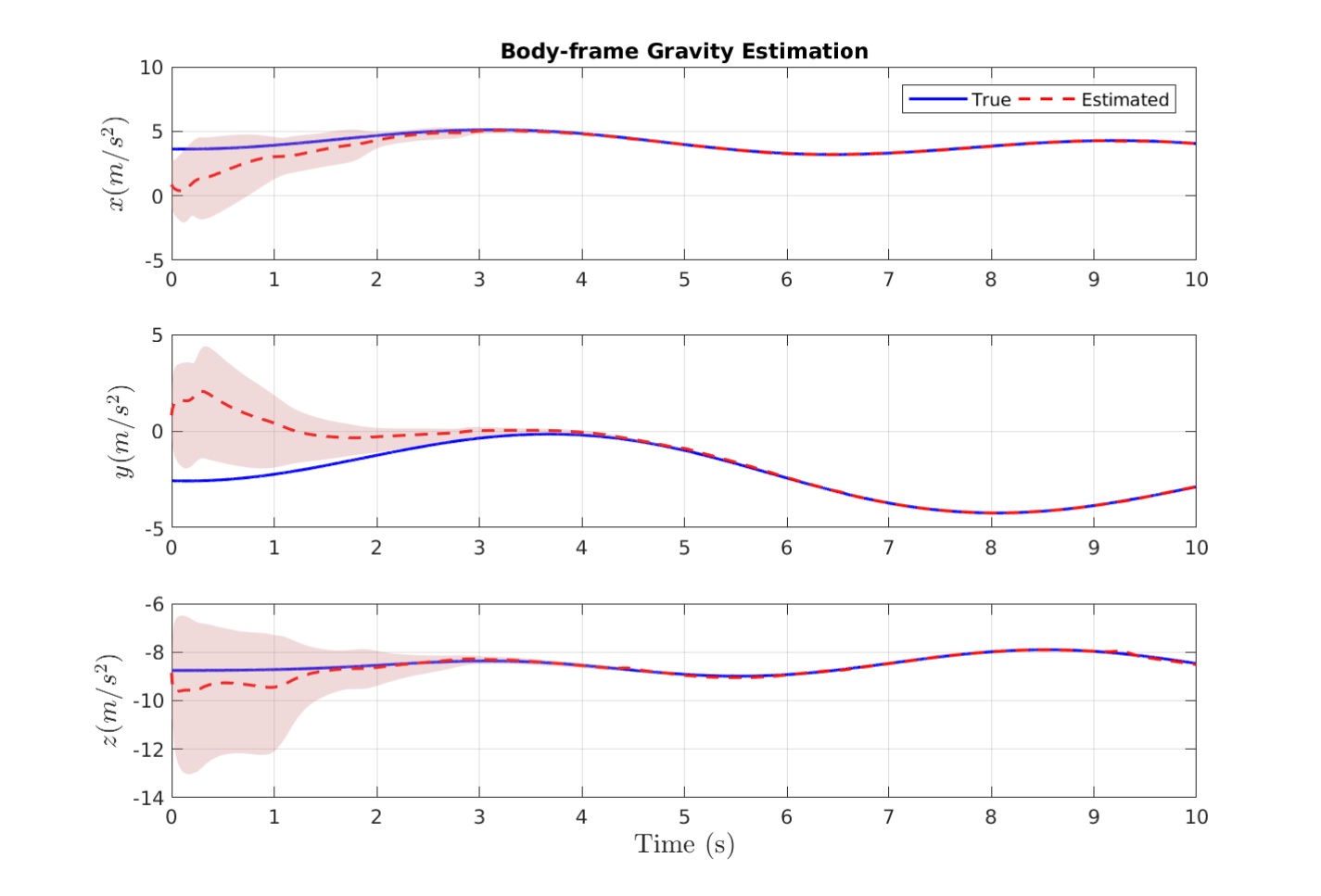}
    \includegraphics[width=1.03\linewidth]{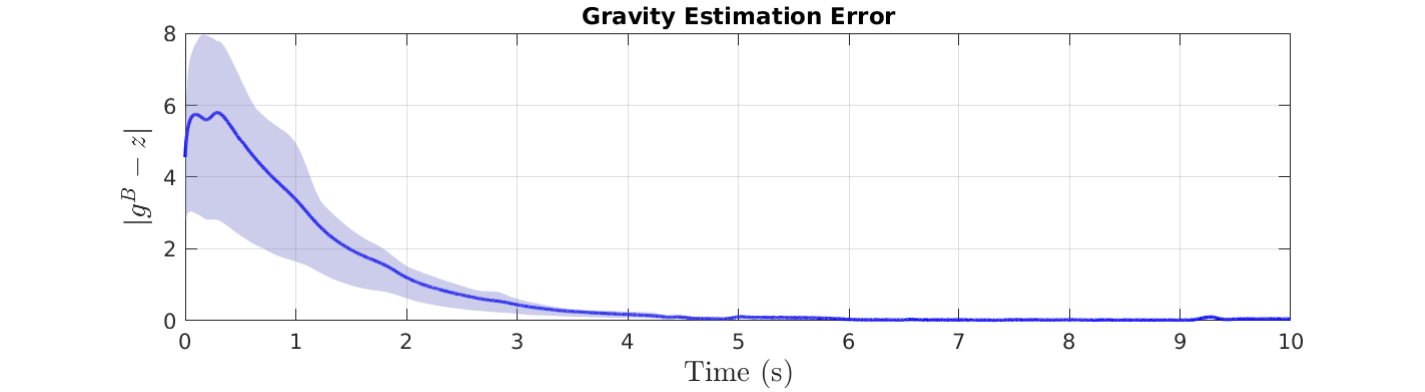}
    \vspace{-.3cm}
    \caption{Estimated and true body-frame gravity components, and observer estimation error $|\bm{g}^\calB - \bm{z} |$. 
    }
    \label{fig:gB_estimation}
\end{figure}
\begin{figure}[h]
    \centering
    \includegraphics[width=1.03\linewidth]{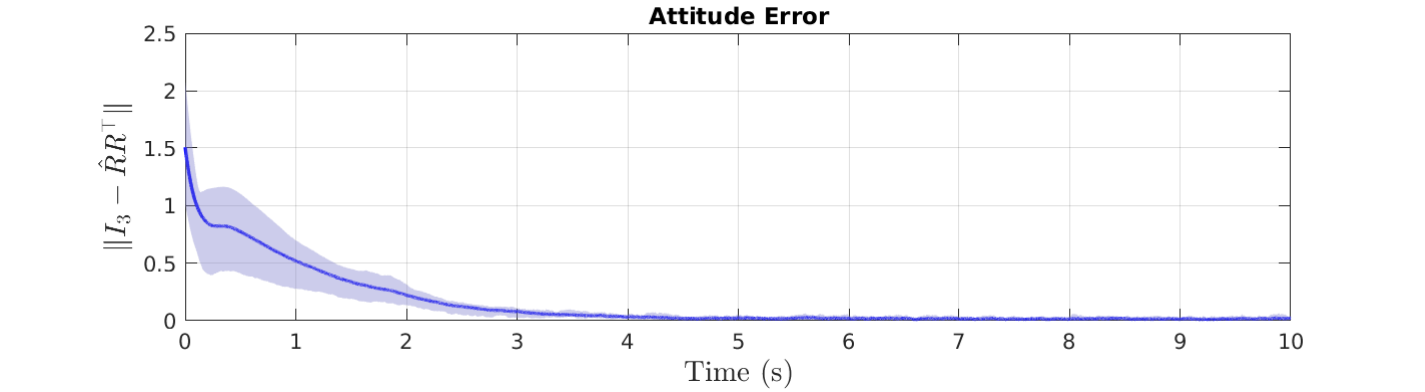}
    \includegraphics[width=1.03\linewidth]{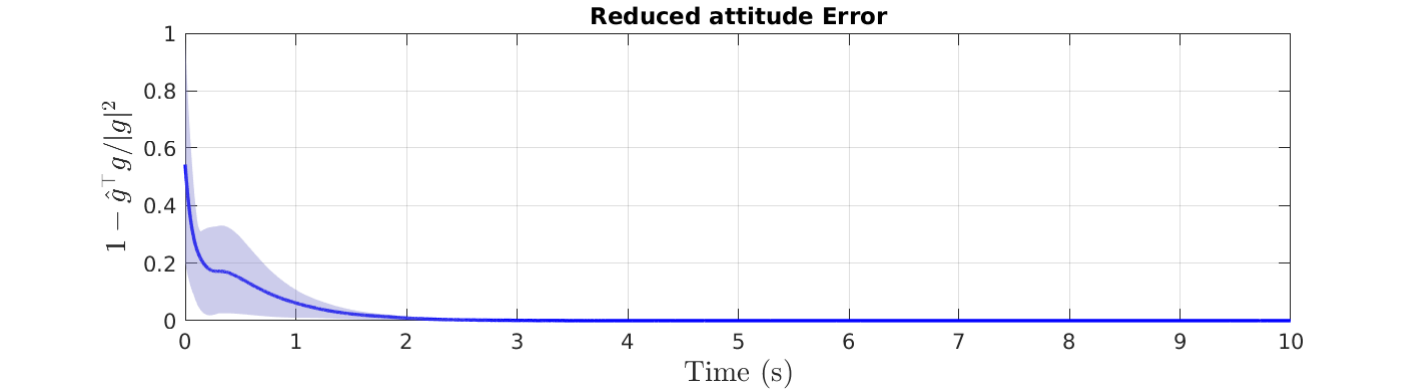}
    \vspace{-.3cm}
    \caption{Full attitude estimation error $\|I_3 - \hat{R} R^\top\|$, and reduced attitude (gravity direction) estimation error $1-\hat{\bm{g}}^\top \bm{g}/|\bm{g}|^2$. 
    }
    \label{fig:R_estimation}
\end{figure}


    \section{Conclusions} \label{sec:conclusion}
In this work, we proposed a cascaded observer for visual-inertial odometry that combines velocity direction and IMU measurements to estimate the navigation states.
A Riccati observer is designed to ensure global exponential stability of velocity estimation under persistently exciting motion, along with a nonlinear complementary filter on $\SO(3)$ that provides an estimate of the attitude using the estimated gravity vector and, optionally, a magnetometer for yaw estimation. The proposed framework achieves almost global asymptotic stability of the interconnection. An iterative pre-observer algorithm is introduced to estimate the velocity direction from sparse optical flow data. Both algorithms were validated in simulation.
Future work involves implementing and benchmarking the proposed solution against state-of-the-art VIO methods, and extending the iterative algorithm to jointly estimate $\bm{\eta}_v$ and $\bm{\omega}$.

\bibliography{ref}

\end{document}